\documentclass[journal,compsoc]{IEEEtran}
\usepackage{amsfonts}
\usepackage{amssymb}
\usepackage{array, makecell}
\usepackage{mathtools}
\usepackage{graphicx}
\usepackage{subfigure}
\usepackage{enumerate}
\usepackage{amsmath}
\usepackage{color}
\usepackage{amsthm}
\usepackage[colorlinks,linkcolor=black,anchorcolor=black,citecolor=black]{hyperref}
\usepackage{algorithm}
\usepackage{algpseudocode}
\usepackage{stfloats}
\usepackage{multirow}
\theoremstyle{definition}
\usepackage[bottom=1.18in,left=0.65in,right=0.7in,top=0.67in]{geometry}
\usepackage{color}
\usepackage{url}
\allowdisplaybreaks[4]

\ifCLASSOPTIONcompsoc
  \usepackage[nocompress]{cite}
\else
  \usepackage{cite}
\fi

\ifCLASSINFOpdf

\else

\fi

\begin{document}
\newtheorem{theorem}{Theorem}
\newtheorem{lemma}{Lemma}
%

\title{Online LLM Selection via Constrained Bandits with Time-Varying Demand}

\author{Yin Huang ~\IEEEmembership{Student Member,~IEEE}
           ~~Qingsong Liu~\IEEEmembership{Member,~IEEE}
       ~~Jie~Xu~\IEEEmembership{Senior Member,~IEEE}
\IEEEcompsocitemizethanks{\IEEEcompsocthanksitem Y. Huang and J. Xu are with the Department of Electrical and Computer Engineering, University of Florida.
Email: \{yin.huang, jie.xu\}@ufl.edu.

\IEEEcompsocthanksitem Q. Liu is with the Manning College of Information and Computer Sciences, University of Massachusetts Amherst. Email: qingsongliu@umass.edu.
}}

\IEEEtitleabstractindextext{
\begin{abstract}
Large Language Models (LLMs) are increasingly deployed in edge–cloud inference systems to handle diverse user tasks with heterogeneous accuracy, latency, and cost profiles. Selecting the appropriate LLM for each incoming task is critical for ensuring service quality and efficient resource utilization. However, model heterogeneity, stochastic and unknown performance characteristics, and time-varying task demands make static selection strategies inadequate. Real-world deployments often impose hard resource budgets such as monetary expenditure limits, along with soft service-level requirements such as latency guarantees. These constraints introduce additional challenges for online decision-making. We formulate this problem as a constrained stochastic bandit learning task, where the learner sequentially selects models under both packing-type (hard) and covering-type (soft) constraints, while adapting to time-varying task demand. The learner operates without access to the underlying reward, cost, or latency distributions and must rely on partial feedback. We develop a novel online learning algorithm that leverages confidence-bound estimates and demand predictions to balance reward maximization with long-term constraint satisfaction. We provide theoretical guarantees showing sublinear regret and sublinear covering constraint violations compared to an offline benchmark with full information. Experimental results on synthetic workloads demonstrate the effectiveness and robustness of our approach in dynamic, resource-constrained environments.
\end{abstract}
}

\maketitle
\IEEEdisplaynontitleabstractindextext
\IEEEpeerreviewmaketitle
\IEEEraisesectionheading{\section{Introduction}}
\IEEEPARstart{L}arge language models (LLMs) are increasingly deployed in real-world systems to support tasks such as generation, summarization, and question answering. In collaborative edge-cloud inference systems, a mobile device offloads incoming requests to a cloud server, where it can select from a pool of pre-deployed LLMs. A central challenge is choosing the most appropriate model to serve incoming tasks, given the heterogeneity in model capabilities. Most existing approaches rely on static selection, fixing one LLM throughout deployment~\cite{owodunni2023koya,peng2023check}. However, in new or evolving scenarios, such one-size-fits-all approach often underperform. Due to generation diversity~\cite{chen2023frugalgpt}, smaller models may be better suited for certain prompts, while data drift~\cite{bhardwaj2022ekya} may alter task characteristics over time. These issues call for an online model selection~\cite{xia2024llm} approach that adapts to the current workload. Importantly, different LLMs often exhibit varying and uncertain accuracy, latency, and monetary cost. This makes the online model selection problem a natural fit for a bandit learning framework, where the system must balance exploration and exploitation while adapting to stochastic reward and cost feedback.

An additional challenge in online LLM selection lies in satisfying multiple resource constraints, which often vary in both form and severity. These constraints can be broadly classified along two dimensions. First, a constraint can be hard, meaning it must be satisfied at all times, or else the system halts or soft, meaning occasional violations are allowed as long as performance remains acceptable on average. Second, constraints can be of packing or covering type: a packing constraint enforces that cumulative resource usage does not exceed a fixed budget (i.e., $\le$), while a covering constraint requires that cumulative service provision meets a minimum threshold (i.e., $\ge$)~\cite{li2019combinatorial,slivkins2023contextual}. Our problem involves both types. First, LLMs typically incur monetary cost based on the number of tokens they generate this cost is stochastic, as token counts vary across tasks and models~\cite{sun2025invisible,husom2025sustainable}. In many deployments, users or providers face a fixed budget, and exceeding it may terminate service. This defines a hard packing constraint. Second, many applications impose a latency Service-Level Agreement (SLA), which demands that a certain fraction of tasks be completed within a target deadline~\cite{yu2022deterministic}. Occasional violations are tolerable, but too many missed deadlines degrades user experience. This defines a soft covering constraint. These bidirectional constraints reflect fundamentally different goals, cost containment vs. service quality, and demand tailored algorithmic handling.

A third challenge arises from the dynamic and uncertain nature of task arrival rates across time. In edge-cloud inference systems, the number of tasks arriving in each time slot may vary significantly due to user activity, temporal patterns, or external conditions~\cite{balseiro2023online}. To ensure low overhead and reduce cold-start delays, a single model is typically selected at the start of each slot to serve all tasks that arrive during that interval. While this preselection avoids per-task switching costs, it makes model decisions particularly sensitive to variations in task volume. Ignoring such fluctuations can lead to degraded performance: lightweight load may tolerate slower or cheaper models, whereas heavy load may require faster, more costly ones to meet latency constraints.  Fortunately, forecast information is often available in practice for example, from historical workload traces or predictive demand models and can be leveraged to anticipate demand sizes before decision-making. Integrating such predictions into the learning algorithm allows more informed and robust model selection as conditions change, particularly under budget constraints or time-sensitive SLAs.

In this paper, we formulate the online LLM model selection problem for a task offloading scenario in a heterogeneous edge-cloud inference system, explicitly addressing the aforementioned challenges. We consider a setting with multiple available LLMs (e.g., different model sizes or providers at edge and cloud), where tasks arrive continuously over time and must be assigned to appropriate models for inference. Our formulation captures three key aspects: (i) the need to dynamically select among candidate models at each decision point rather than committing to a static allocation, (ii) the existence of a hard budget constraint on token-based inference costs alongside a soft latency SLA that ensures quality of service, and (iii) the presence of time-varying task volumes that influence both cost and responsiveness. By integrating these elements, the proposed framework generalizes classical offloading and scheduling formulations to the setting of LLM inference, where monetary efficiency and latency guarantees are primary design considerations.

To solve this online model selection problem, we cast it as a constrained stochastic multi-armed bandit instance, akin to the bandits with knapsacks framework, where the system faces both a hard budget constraint and a soft latency SLA, corresponding to packing- and covering-type constraints, respectively.We develop a new learning algorithm, COPAC-UCB, which integrates confidence-guided estimation, being optimistic (via upper confidence bounds or UCB) for reward and covering-type constraints, and pessimistic (via lower confidence bounds, or LCB) for packing-type constraints with Lagrangian regularization to adaptively balance performance and feasibility under uncertainty. To improve responsiveness under demand variability, COPAC-UCB incorporates a forecasting component that estimates the cumulative task load using a black-box predictor. These demand estimates help augment decisions to more effectively satisfy constraints. Meanwhile, the algorithm updates dual variables (virtual prices) for each constraint via online gradient descent, effectively discouraging models that risk violating the budget or SLA. By combining predictive guidance with optimistic exploration and adaptive cost signals, COPAC-UCB balances reward maximization with safe resource usage over time. The key contributions of this paper are summarized as follows:
\begin{itemize}
\item We formulate a new online LLM model selection problem in edge-cloud collaborative inference systems, explicitly incorporating time-varying task demand, monetary and latency constraints, and model heterogeneity.
\item We cast this problem as a stochastic bandit with both packing and covering constraints, and propose a novel algorithm, COPAC-UCB, which combines confidence-bound estimates with Lagrangian-based resource balancing. The algorithm also leverages predictions of future task load to guide cost-aware model selection under dynamic demand.
\item We theoretically show that COPAC-UCB achieves sublinear regret while satisfying budget and latency constraints with high probability, ensuring safe and efficient decision-making over time.
\item We evaluate our algorithm through simulations on realistic workload traces, demonstrating that COPAC-UCB achieves higher utility and better latency compliance under budget than several competitive baselines.
\end{itemize}

\section{Related Work}

\textbf{Inference Selection and Offloading Optimization} Recent years have seen growing interest in improving LLM inference efficiency via adaptive selection and offloading. Traditional methods~\cite{owodunni2023koya,peng2023check,salazar2019masked,huang2025thriftllm} often adopt a static model, chosen offline by benchmark scores or perplexity, but ignore input-dependent variability across LLMs. More recent work proposes online learning-based model selection to balance accuracy, latency, and monetary cost under uncertainty. For example,~\cite{xia2024llm} models dynamic performance drift via a time-varying bandit framework. In parallel, offloading methods decide whether to run tasks locally, at the edge, or in the cloud, optimizing for latency, energy, or budget~\cite{liu2024resource,al2024regret,yuan2025local,dai2024cost,ling2024multi,yang2025generalizable,dai2023offloading,huang2023adversarial,huang2025learning}. Among them,~\cite{dai2024cost} introduces a combinatorial bandit approach for LLM placement under cost constraints. These lines of work underscore the value of principled online decision-making for LLM inference in dynamic, resource-constrained settings. However, none of these works jointly address both packing and covering constraints or explicitly incorporate time-varying demand, which are central to our formulation.

\textbf{Bandits with Arm Selection Constraints}
Bandits with arm selection constraints model online decision-making where each action yields a reward while consuming resources or satisfying additional constraints. A canonical example is the Bandits with budget constraint framework~\cite{badanidiyuru2018bandits}, where each arm incurs stochastic rewards and resource consumption, and the goal is to maximize cumulative reward before exhausting a fixed budget. Subsequent work~\cite{agrawal2019bandits} developed primal-dual and UCB-based algorithms with sublinear regret guarantees under stationary settings. More recent extensions introduce covering-type constraints to capture fairness or service-level guarantees. For instance,~\cite{joseph2016fairness,gillen2018online,xu2020combinatorial} enforce long-term proportionality across arms or coverage of service across groups, modeled as soft constraints that require cumulative satisfaction over time. These settings typically balance reward maximization with constraint violation minimization. Other lines of work explore adversarial~\cite{immorlica2022adversarial,huang2024adversarial}, non-stationary~\cite{cheung2019learning}, and time-varying settings with oracle predictions~\cite{lyu2023bandits,huang2024quantum}, showing that side information can improve performance. However, existing models rarely address both packing and covering constraints simultaneously, and do not consider time-varying, stochastic demand. Our work fills this gap by formulating and analyzing a setting that integrates all three dimensions. 

\textbf{Learning-augmented online algorithm design}
Our work connects to the literature on learning-augmented online algorithms, which leverage machine-learned predictions to enhance online decision-making \cite{lykouris2021competitive, mitzenmacher2022algorithms}. These methods improve performance by incorporating predictions on latent parameters, while maintaining robustness to adversarial inputs. Applications span caching \cite{lykouris2021competitive}, scheduling \cite{mitzenmacher2022algorithms, lattanzi2020online}, rent-or-buy \cite{purohit2018improving}, set cover \cite{bamas2020primal, almanza2021online}, and matching \cite{antoniadis2020secretary}, often assuming a static prediction known at the start of the horizon. Recent work explores sequential predictions updated at each round \cite{rakhlin2013online, rakhlin2013optimization, steinhardt2014adaptivity, jadbabaie2015online}, primarily in full-feedback or contextual bandit settings. However, these models do not incorporate resource constraints. We extend this paradigm by combining sequential predictions with mixed hard and soft constraints, introducing new algorithmic challenges.

\section{System Model}
We consider a mobile edge–cloud inference system where a user device offloads LLM inference tasks to a cloud-based service. Due to limited on-device compute and memory, the device cannot run large-scale LLMs locally and relies on the cloud for inference. The system proceeds in discrete rounds indexed by \( t = 1, 2, \dots, T \), each representing a fixed-duration time slot. At the start of each round $t$, the system selects a single LLM \( a_t \in \mathcal{A} \) from a pool of pre-deployed models to handle tasks arriving during the slot. This preselection avoids frequent model switches and reduces cold-start overhead. The task load \( q_t \) in round \( t \) is unknown at selection time and may fluctuate with user activity or application demand. Although the learner must adapt without knowing \( q_t \) in advance, such variations often follow structured patterns (e.g., time-of-day or seasonal trends) and can be predicted from historical data or external signals~\cite{rakhlin2013online,rakhlin2013optimization}. These predictions help guide model selection and resource planning.

\textbf{Model Performance and Cost.}
Each model \( a \in \mathcal{A} \) in the cloud has a stochastic cost and performance profile that is only revealed after it is selected. In each round \( t \), the system selects a model \( a_t \) to serve tasks arriving during the time slot. The selected model then yields a random per-task reward \( r_t(a_t) \in [0,1] \), capturing task accuracy or success rate, and a latency \( \ell_t(a_t) \), which includes model inference time and system delay. Recall that \( q_t \) is the received number of tasks during round \( t \), the total reward and latency incurred in this round are \( q_t \cdot r_t(a_t) \) and \( q_t \cdot \ell_t(a_t) \), respectively.

Each model \( a \) incurs a monetary cost based on a fixed per-token price \( \rho_a > 0 \), specified by the service provider. The number of tokens generated per task varies with prompt complexity and decoding behavior. Let \( \tau_t^i(a) \) denote the token usage of model \( a \) on task \( i \) in round \( t \), drawn from a bounded distribution with unknown mean~\cite{husom2025sustainable}. The total token usage in round \( t \) is \( \tau_t(a) = \sum_{i=1}^{q_t} \tau_t^i(a) \), and the corresponding cost is \( \rho_{a_t} \cdot \tau_t(a_t) \), revealed only after execution.  The learner must choose models without prior knowledge of realized reward, latency, or token usage, relying on historical feedback \( (r_t(a_t), \ell_t(a_t), \tau_t(a_t)) \) to balance accuracy, latency, and cost.

\textbf{Constraints.}
The system operates under two global constraints that govern model selection throughout the decision horizon. First, a hard monetary budget constraint limits the cumulative token-based cost. Formally, the total incurred cost must satisfy $\sum_{t=1}^{T} \rho_{a_t} \cdot \tau_t(a_t) \leq B$. This reflects operational or financial limits, such as API usage quotas or cloud usage budgets. Once the budget is exhausted, meaning the cumulative cost exceeds \(B\), the system is prohibited from using any LLMs. It continues to process incoming workloads until round \(T\), but can only select a default no-operation action that produces zero reward and incurs zero cost. This represents a fallback mode where the system remains functional but consumes no additional resources.

Second, a soft latency service-level agreement imposes a covering constraint on user-perceived responsiveness~\cite{yu2022deterministic}. At least a fraction \( \alpha \in (0,1] \) of all tasks must complete within a latency deadline \( d_0 \). Since all \( q_t \) tasks in round \( t \) share the same model latency \( \ell_t(a_t) \), the SLA constraint is: $\sum_{t=1}^{T} q_t \cdot \mathbf{1}\{\ell_t(a_t) \leq d_0\} \geq \alpha \cdot \sum_{t=1}^{T} q_t.$ The indicator is treated as zero when a no-operation action is selected, ensuring only effective decisions contribute to SLA compliance. While occasional violations are permitted, persistent delays significantly degrade user experience. As latency is stochastic and only observed after selection, the learner must balance responsiveness against reward and cost. See Fig.~1 for a system overview.

\textbf{Objective.}
The goal of the system is to maximize the expected cumulative user-perceived reward over a fixed decision horizon of \(T\) rounds, while satisfying both a hard monetary budget constraint and a soft latency SLA. Formally, the learner seeks to solve the following constrained online optimization problem:
\[
\begin{aligned}
\max_{a_1, \dots, a_T} \quad & \mathbb{E} \left[ \sum_{t=1}^{T} q_t \cdot r_t(a_t) \right] \\
\text{subject to} \quad 
& \sum_{t=1}^{T} \rho_{a_t} \cdot \tau_t(a_t) \leq B, \quad \text{(monetary budget)} \\
\sum_{t=1}^{T} q_t \cdot &\mathbf{1}\{\ell_t(a_t) \leq d_0\} \geq \alpha \cdot \sum_{t=1}^{T} q_t.  \text{ (latency SLA)}
\end{aligned}
\]
As described earlier, any no-operation decision contributes zero reward and is excluded from SLA satisfaction.

The expectation is taken over both the learner’s internal randomness and the stochasticity of reward, latency, and token generation. This formulation captures the core trade-off between reward maximization and constraint satisfaction under uncertainty. Since the performance and cost profiles of each model are unknown and revealed only through interaction, the problem aligns naturally with a constrained bandit framework. The learner must explore model behavior, adapt to time-varying task loads, and make decisions under partial feedback, all while respecting the budget and service constraints.

\begin{figure}[tb]
    \centering
    \includegraphics[width=0.5\linewidth]{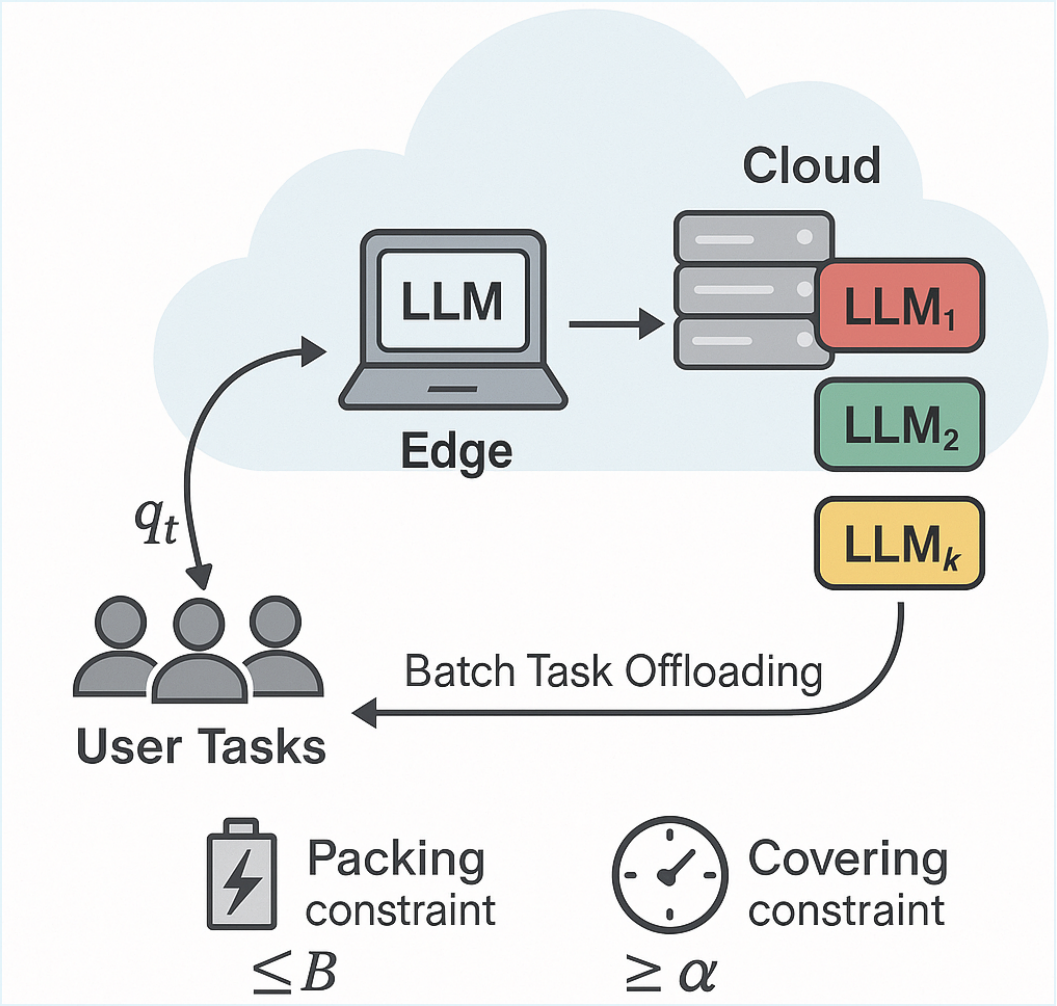}
    \caption{Online LLM selection under Budget Constraint} \label{fig:LLM-trends}
\end{figure}

\section{Time Varying Demands Bandit Formulation with General Constraints.}
We formulate the model selection problem as a Time Varying Demands Bandit Formulation with General Constraints. While this framework is motivated by the above model selection problem, it is more general and can be applied to a broader class of constrained online decision-making problems with time-varying demand, making it of independent interest. In each round \( t \in \{1, \dots, T\} \), the learner selects an action \( a_t \in \mathcal{A} \) from a finite set of LLM models to process the incoming tasks. The learner then observes a stochastic reward \( r_t(a_t) \in [0,1] \), capturing average task utility such as accuracy or success rate.

While our system specifically involves inference latency and monetary cost, many practical deployments also face other resource pressures, such as energy usage, memory access, or bandwidth limitations. To capture this broader class of constraints in a unified manner, we introduce a \( d \)-dimensional cost vector \( \mathbf{z}_t(a_t) \in [0,1]^d \), where each coordinate represents the per-task consumption of a specific resource type. The total cost in round \( t \) is scaled by the realized task load as \( q_t \cdot \mathbf{z}_t(a_t) \). This abstraction allows us to handle heterogeneous system costs within a single learning framework and sets the stage for constraint-aware decision-making.

\textbf{Constraint Structure.}
We categorize the \( d \)-dimensional cost vector \( \mathbf{z}_t(a_t) \in [0,1]^d \) into two types of constraints: packing constraints \( \mathcal{I}_{\text{pack}} = \{1, \dots, d_1\} \) with hard budgets \( B_i > 0 \); and covering constraints \( \mathcal{I}_{\text{cov}} = \{d_1 + 1, \dots, d\} \) with soft lower thresholds, where \( d = d_1 + d_2 \). For each packing constraint \( i \in \mathcal{I}_{\text{pack}} \), the system halts if the cumulative usage exceeds budget \( B_i \) at any round \( \tau \), i.e., \( \sum_{t=1}^{\tau} q_t \cdot z_t^{(i)}(a_t) > B_i \). For each covering constraint \( i \in \mathcal{I}_{\text{cov}} \), the learner must ensure the cumulative usage satisfies \( \sum_{t=1}^{T} q_t \cdot z_t^{(i)}(a_t) \geq B_i \). For instance, the latency SLA introduced earlier can be represented as a special case where $ z_t^{(i)}(a_t) := \mathbf{1}\{\ell_t(a_t) \leq d_0 \}$, and  $B_i = \alpha \cdot \sum_{t=1}^{T} q_t$. 

\textbf{Objective.} The learner does not observe the underlying reward or cost distributions, nor the future demand sequence $\{q_t\}_{t=1}^T$. It must adaptively select actions based only on past observations. The goal is to maximize the total cumulative reward over rounds $\{1, \dots, \tau-1\}$, where $\tau$ denotes the first round in which any packing constraint is violated, while ensuring all covering constraints are satisfied over the horizon.

\textbf{Benchmark.} To evaluate algorithm performance, we benchmark against a clairvoyant offline policy that knows all latent distributions $\{r_t(a), z_t^{(i)}(a)\}$ and the entire demand sequence $\{q_t\}_{t=1}^T$ in advance. Let $\text{OPT}$ denote the expected total reward achieved by this dynamic oracle policy, which can adaptively choose the best action in each round using full information, including future outcomes. Given an online learner that selects actions $a_t$ without access to latent reward or cost distributions, we define the regret at horizon $T$ as the gap between the offline dynamic oracle and the learner’s cumulative expected reward:
\begin{equation}
    \text{Regret}(T) := \text{OPT} - \mathbb{E}\left[ \sum_{t=1}^{\tau - 1} q_t \cdot r_t(a_t) \right],
\end{equation}
where  $\tau$ is the first round that violates any packing constraint.

To facilitate analysis, we upper bound $\text{OPT}$ by introducing a linear programming (LP) relaxation based on a static randomized policy. Let $\mathbf{u} \in \Delta_{|\mathcal{A}|}$ be a fixed probability distribution over actions, $\mathbf{r} \in \mathbb{R}^{|\mathcal{A}|}$ the vector of expected per-task rewards (i.e. $\mathbf{r} = \mathbb{E}[r_t(a)]$), and $\mathbf{z}^{(i)} \in \mathbb{R}^{|\mathcal{A}|}$ the expected per-task consumption along dimension $i \in [d]$ (i.e., $\mathbf{z}^{(i)}(a) = \mathbb{E}[z_t^{(i)}(a)]$). Let $Q := \sum_{t=1}^T q_t$ be the total task volume. Then the LP relaxation benchmark is defined as:
\begin{align}
\text{OPT}_{\text{LP}} = &\max_{\mathbf{u} \in \Delta_{|\mathcal{A}|}} \quad  Q \cdot \mathbf{r}^\top \mathbf{u} \nonumber \\
& \text{s.t.} \quad 
 Q \cdot \mathbf{z}^{(i)\top} \mathbf{u} \leq B, \quad \forall i \in \mathcal{I}_{\text{pack}} \nonumber \\
& \qquad\, Q \cdot \mathbf{z}^{(i)\top} \mathbf{u} \geq B, \quad \forall i \in \mathcal{I}_{\text{cov}}, \label{eq:linear-program}
\end{align}

For simplicity, here we assume all per-dimension budgets are equal to $B$. Our formulation and results can be extended to heterogeneous budgets with minimal changes. This LP solution provides an upper bound on the performance of any dynamic policy, as formalized by the following lemma:
\begin{lemma}
$\text{OPT}_{\text{LP}} \geq \text{OPT}.$
\label{lem:lp-upper-bound}
\end{lemma}
\begin{proof}[Proof Sketch]
We begin by considering the relaxed LP form:
\begin{align*}
&\text{OPT}'_{\text{LP}} := \max_{\boldsymbol{x}_t \in \Delta_{|\mathcal{A}|},\ \forall t \in [T]} \sum_{t=1}^T q_t \cdot \boldsymbol{r}^\top \boldsymbol{x}_t \\
&\text{s.t.} \quad \sum_{t=1}^T q_t \cdot (\mathbf{z}^{(i)})^\top \boldsymbol{x}_t \le B, \quad \forall i \in \mathcal{I}_{\text{pack}}, 
\\&
\sum_{t=1}^T q_t \cdot (\mathbf{z}^{(i)})^\top \boldsymbol{x}_t \ge B, \quad \forall i \in \mathcal{I}_{\text{cov}},
\end{align*}
where \( \boldsymbol{x}_t \in \Delta_{|\mathcal{A}|} \) is a probability distribution over actions at time \( t \). Any policy \( \pi \) achieving \( \text{OPT} \) induces a feasible solution to \( \text{OPT}'_{\text{LP}} \) by setting $x_{t,a} := \mathbb{E}_{a_t \sim \pi}[\mathbf{1}\{a_t = a\}]$, which implies \( \text{OPT}'_{\text{LP}} \ge \text{OPT} \). Conversely, any feasible solution \( \{\boldsymbol{x}_t\}_{t=1}^T \) to \( \text{OPT}'_{\text{LP}} \) yields a static distribution $\boldsymbol{u} := \frac{\sum_{t=1}^T q_t \boldsymbol{x}_t}{\sum_{t=1}^T q_t}$, which is feasible for \( \text{OPT}_{\text{LP}} \) and achieves the same value. Hence, \( \text{OPT}'_{\text{LP}} = \text{OPT}_{\text{LP}} \ge \text{OPT} \).
\end{proof}
This benchmark enables tractable regret analysis, while preserving fidelity to the original problem. Any regret bound relative to $\text{OPT}_{\text{LP}}$ directly implies a bound relative to $\text{OPT}$.

\textbf{Constraint Violation Metric.}
To unify the treatment of covering-type constraints, we define the cumulative violation for each covering resource \( i \in \mathcal{I}_{\text{cov}} \) as
\[
V_i(T) := \left[ B - \sum_{t=1}^{T} q_t \cdot z_t^{(i)}(a_t) \right]_+,
\]
The learner aims to simultaneously achieve sublinear regret $\text{Regret}(T)$ and ensure vanishing average violation: $\frac{V_i(T)}{T} = o(1), \forall i \in \mathcal{I}_{\text{cov}}.$ This formulation captures both reward maximization and long-term covering constraint satisfaction under model uncertainty.

\section{Algorithm Design}

We propose COPAC-UCB (Constrained Prediction-Aware Covering and Packing UCB), a learning algorithm for our framework. It selects actions being pessimistic about packing-type costs while optimistic about reward and covering-type costs. In each round, COPAC-UCB computes a UCB for expected reward, an LCB for packing-type costs, and a UCB for covering-type costs. The action is chosen to maximize a regularized utility objective with the reward of each candidate action scaled by constraint amplification multipliers that boost perceived reward when the corresponding constraints have previously been violated. To handle task load variability, COPAC-UCB uses a predicted cumulative demand $\hat{Q}_t$ to approximate the total number of tasks. This estimate normalizes resource usage by divided by total budget $B$, aiding cost-aware decision-making. Its construction is detailed in subsection~\ref{estimation_Qt}.

\subsection{Confidence Bound Estimation}
Let \( N_t(a) \) denote the number of times model \( a \in \mathcal{A} \) has been selected before round \( t \).  We define the empirical mean reward and consumption for each action \( a \) as follows:
$
\hat{r}_t(a) := \frac{1}{N_t(a)} \sum_{s=1}^{t-1} \mathbf{1}\{a_s = a\} \cdot r_s(a),
\quad 
\hat{z}_t^{(i)}(a) := \frac{1}{N_t(a)} \sum_{s=1}^{t-1} \mathbf{1}\{a_s = a\} \cdot z_s^{(i)}(a), \quad \forall i \in [d].
$ To facilitate exploration, we adopt the confidence radius function from~\cite{lyu2023bandits}: $\text{rad}(v, N, \delta) := \sqrt{\frac{2v \log(1/\delta)}{N}} + \frac{4 \log(1/\delta)}{N}.$

We then construct the following confidence bounds:
\begin{align}
\text{UCB}_{r,t}(a) &:= \min\left\{ \hat{r}_t(a) + \text{rad}(\hat{r}_t(a), N_t(a), \delta), \ 1 \right\},\label{ucb_reward} \\
\text{LCB}_{c,t}^{(i)}(a) &:= \max\left\{ \hat{z}_t^{(i)}(a) - \text{rad}(\hat{z}_t^{(i)}(a), N_t(a), \delta), \ 0 \right\},\nonumber\\&\qquad\qquad\qquad\qquad \text{ for } i \in \mathcal{I}_{\text{pack}},\label{lcb_pack} \\
\text{UCB}_{c,t}^{(i)}(a) &:= \min\left\{ \hat{z}_t^{(i)}(a) + \text{rad}(\hat{z}_t^{(i)}(a), N_t(a), \delta), \ 1 \right\},\nonumber\\&\qquad\qquad\qquad\qquad \text{ for } i \in \mathcal{I}_{\text{cov}}.\label{ucb_cover}
\end{align}


These confidence intervals ensure, with high probability, that the true reward and costs lie within the bounds. The following guarantee formalizes this:
\begin{lemma}
With probability at least $1 - 3KTd\delta$, for all $a \in \mathcal{A}$, we have $\text{UCB}_{r,t}(a) \ge r(a),$ and $i \in \mathcal{I}_{\text{pack}}, \text{LCB}_{c,t}^{(i)}(a) \le z^{(i)}(a),$ and $i \in \mathcal{I}_{\text{cov}}, \text{UCB}_{c,t}^{(i)}(a) \le z^{(i)}(a)$.
\end{lemma}
\begin{proof}[Proof Sketch]
The result follows from the argument of~\cite{babaioff2015dynamic,agrawal2014bandits} by applying concentration bounds to the empirical means of rewards and costs. With high probability, the constructed UCB and LCB terms respectively upper and lower bound the true values, completing the proof.
\end{proof}

\subsection{Prediction of Task Demands}\label{estimation_Qt}
To accommodate time-varying task demand, the learner employs a time-series predictor $\mathcal{P}$ to estimate the total task volume $Q := \sum_{t=1}^T q_t$ from historical data. At each round $t$, the learner observes $\{q_s\}_{s=1}^{t-1}$ and produces a forecast $\hat{Q}_t = \mathcal{P}_t(q_1, \dots, q_{t-1})$. In practice, $\mathcal{P}$ can be any black-box model that supports one-step or multi-step prediction. For example, it may recursively generate forecasts $\hat{q}_t, \hat{q}_{t+1}, \dots, \hat{q}_T$ using past and previously predicted values, and define $\hat{Q}_t := \sum_{s=1}^{t-1} q_s + \sum_{s=t}^T \hat{q}_s$ as the estimated total demand. These predictors can incorporate statistical techniques or deep models widely used in time-series forecasting~\cite{shumway2017time,hyndman2021forecasting,lim2021time}.

We define the prediction error at round $t$ as $\varepsilon_t := |Q - \hat{Q}_t|$, which captures the deviation between true and estimated cumulative demand. This estimate is used by the learner to adapt cost-aware decisions to upcoming load variability, enabling dynamic adjustment under uncertainty. Without loss of generality, we also assume that \( q_t \le \bar{q} \) for all \( t \), where \( \bar{q} \) is a known constant upper bound on task load. This upper bound enables the learner to gett worst-case cost scaling and construct confidence intervals that account for demand uncertainty.

\subsection{Algorithm Description}
Algorithm~1 presents the COPAC-UCB algorithm. At each round $t \in [T]$, the learner uses a predicted cumulative task volume $\hat{Q}_t$ from a forecasting oracle $\mathcal{P}_t$, based on historical demands $\{q_s\}_{s=1}^{t-1}$. Using this forecast, COPAC-UCB constructs high-probability confidence intervals for reward and per-task resource consumption.

Specifically, for each action \( a \in \mathcal{A} \), COPAC-UCB computes a UCB \( \text{UCB}_{r,t}(a) \) on the expected per-task reward. For each resource dimension \( i \in [d] \), it constructs a one-sided confidence bound on the unit resource consumption \( z_t^{(i)}(a) \): a lower bound \( \text{LCB}_{c,t}^{(i)}(a) \) if \( i \in \mathcal{I}_{\text{pack}} := [1, d_1] \), and an upper bound \( \text{UCB}_{c,t}^{(i)}(a) \) if \( i \in \mathcal{I}_{\text{cov}} := [d_1 + 1, d] \).

We define a unified confidence vector for each action \( a \in \mathcal{A} \):
\begin{equation}
    \left[\boldsymbol{z}_t^{\text{conf}}(a)\right]_i :=
\begin{cases}
\text{LCB}_{c,t}^{(i)}(a), & \text{if } i \in \mathcal{I}_{\text{pack}} \\
-\text{UCB}_{c,t}^{(i)}(a), & \text{if } i \in \mathcal{I}_{\text{cov}}
\end{cases}
\quad \in \mathbb{R}^d.
\label{unified_confidence}
\end{equation}

Based on these estimates, the algorithm selects an action \( a_t \in \mathcal{A} \) that maximizes:
\begin{equation}
    \label{eq-selection-rule}
    a_t = \arg\max_{a \in \mathcal{A}} \left\{
\text{UCB}_{r,t}(a)
- \boldsymbol{\lambda}_t^\top \left( \frac{\hat{Q}_t}{B} \cdot \boldsymbol{z}_t^{\text{conf}}(a) \right)
\right\},
\end{equation}
where $\boldsymbol{\lambda}_t \in \mathbb{R}^d$ is the dual price vector (i.e., Lagrange multipliers) maintained by the algorithm and defined later.

This Lagrangian-regularized rule balances reward maximization and constraint satisfaction, penalizing consumption of scarce packing resources via lower bounds, while encouraging usage of underutilized covering resources via upper bounds, approximating the dual form of the static LP in Eq.~\eqref{eq:linear-program}, where dual weights $\boldsymbol{\lambda}_t \in \mathbb{R}^d$ encode the opportunity cost of constrained resources. Packing resource usage is penalized via LCBs to avoid overestimation, while covering resource usage is incentivized via UCBs to promote conservative compliance. The normalization factor $\hat{Q}_t/B$ reflects the marginal value of a unit resource under predicted load.

After choosing $A_t$, the system observes the realized demand size $q_t$, receives reward $q_t \cdot r_t(a_t)$, and consumes $q_t \cdot z_t^{(i)}(a_t)$ units of each resource $i \in [d]$. If any packing resource exceeds its budget, i.e., if $\sum_{s=1}^{t} q_s \cdot z_s^{(i)}(a_s) > B \text{ for some } i \in \mathcal{I}_{\text{pack}},$ then the algorithm halts and pulls a null arm $a_0$ for all subsequent rounds.

To adaptively track the cost of each resource, COPAC-UCB maintains a dual weight vector \( \boldsymbol{\lambda}_t \in \mathbb{R}^d \) that is updated via Online Gradient Descent. At each round \( t \), it solves the convex minimization problem:
\begin{equation}
    \boldsymbol{\lambda}_{t+1} = \Pi_{\mathcal{S}} \left( 
\boldsymbol{\lambda}_t - \eta_t \cdot \nabla f_t(\boldsymbol{\lambda}_t)
\right),
\label{eq-dual-update}
\end{equation}
where the projection \( \Pi_{\mathcal{S}} \) is onto the feasible set \( \mathcal{S} := \{ \boldsymbol{\lambda} \in \mathbb{R}^d : \| \boldsymbol{\lambda} \|_1 \leq T^{1/4},\ \boldsymbol{\lambda} \geq \textbf{0}_d \} \), and the objective function is defined using the signed confidence vector
\( \boldsymbol{z}^{\text{conf}}_t(a_t) \) as:
\[
f_t(\boldsymbol{\lambda}) = \frac{q_t \hat{Q}_t}{B} \cdot \left( 
\left( 
\frac{B}{\hat{Q}_t} \cdot \boldsymbol{s} - \boldsymbol{z}^{\text{conf}}_t(a_t)
\right)^\top \boldsymbol{\lambda} 
\right),
\]
where \( \boldsymbol{s} \in \mathbb{R}^d \) is a direction vector with components
\[
[\boldsymbol{s}]_i :=
\begin{cases}
1, & \text{if } i \in \mathcal{I}_{\text{pack}} \\
-1, & \text{if } i \in \mathcal{I}_{\text{cov}}
\end{cases}.
\]

This dual update encourages the algorithm to prioritize constraint satisfaction by dynamically adjusting weights based on observed feedback and demand forecasts.

This update increases the dual weight \( \boldsymbol{\lambda}_t(i) \) for packing resources \( i \in \mathcal{I}_{\text{pack}} \) when consumption is high (i.e., \( \text{LCB}_{c,t}^{(i)}(a) \) is large), discouraging future overuse. Conversely, it increases \( \boldsymbol{\lambda}_t(i) \) for covering resources \( i \in \mathcal{I}_{\text{cov}} \) when usage is low (i.e., \( \text{UCB}_{c,t}^{(i)}(a) \) is small), thereby encouraging future allocation. This dual feedback mechanism enables the learner to balance reward maximization and long-term constraint satisfaction.

\begin{algorithm}[htb]
	\caption{COPAC-UCB}
	\begin{algorithmic}[1]
	    \State \textbf{Initialize} $\boldsymbol{\lambda}_1 = \frac{1}{d} \boldsymbol{1} \in \mathbb{R}^{d}$, $\eta_t = \frac{2}{M\sqrt{t}}$, $M = (\bar{q}+\frac{\bar{q}^2}{b})$, $\boldsymbol{B} = (B)_{i \in [d]}$.
	    
        \For{$t = 1, 2, \dots, T$}
            \State Use demand predictior $\hat{Q}_t = \mathcal{P}_t(q_1, \ldots, q_{t-1})$.
            
\State For all $a \in \mathcal{A}$, compute $\text{UCB}_{r,t}(a)$ using (\ref{ucb_reward});
\State for $i \in \mathcal{I}_{\text{pack}}$, compute $\text{LCB}_{c,t}^{(i)}(a)$ using (\ref{lcb_pack}); 
\State for $i \in \mathcal{I}_{\text{cover}}$, compute $\text{UCB}_{c,t}^{(i)}(a)$ using (\ref{ucb_cover}). 
            
            \State Select action $a_t$ according to \eqref{eq-selection-rule}.

            \If{$\exists i \in \mathcal{I}_{\text{pack}}$ such that $\sum_{s=1}^{t} q_s \cdot z_s^{(i)}(a_s) > B$}
                \State \textbf{break} and pull null arm $a_0$ all the way.
            \EndIf

            \State Observe $q_t$, realized reward $q_t{r}_t(a_t)$, and cost $q_t{z}_t^{(i)}(a_t)$ for all $i \in [d]$.
            
            \State Update budgets: $B \leftarrow B - q_t \cdot {z}_t^{(i)}(a_t)$, for $i \in \mathcal{I}_{\text{pack}}$.
            
            \State Update dual vector: $\boldsymbol{\lambda}_{t+1}$ using Eq.~\eqref{eq-dual-update}.
        \EndFor
	\end{algorithmic}
	\label{alg:COPAC-UCB}
\end{algorithm}

\section{Main Results}
In this section, we give the theoretical analysis of COPAC-UCB. Let \( \boldsymbol{u}^* \in \Delta^{|\mathcal{A}|} \) be an optimal solution to the LP problem \eqref{eq:linear-program}. Before giving an upper bound on the cumulative regret and constraint violation of COPAC-UCB, we first begin with several auxiliary lemmas.

\begin{lemma}[Cumulative Arm Selection Gap with OGD Bound]
\label{lem:ogd-gap}
Define the residual term:
\[
R_t := \frac{q_t\widehat{Q}_t}{B} \cdot \boldsymbol{\lambda}_t^\top 
\left( \frac{B}{\widehat{Q}_t} \cdot \boldsymbol{s} - \boldsymbol{z}_t^{\text{conf}}(\boldsymbol{u}^*) \right),
\]

Then, by applying Online Gradient Descent with step size \( \eta_t = \frac{2}{M\sqrt{t}} \), we have with probability at least \( 1 - 3KT\delta \):
\begin{align*}
    &\sum_{t=1}^{\tau-1} q_t \cdot \boldsymbol{\mathrm{UCB}_{r,t}}^\top \boldsymbol{u}^* - \sum_{t=1}^{\tau-1} q_t \cdot \mathrm{UCB}_{r,t}(a_t) +\sum_{t=1}^{\tau -1}R_t \\&
\leq 
 \sum_{t=1}^{\tau -1}f_t(\boldsymbol{\lambda})+ O\left(M T^{3/4}\right), \quad \forall \boldsymbol{\lambda} \in S.
\end{align*}
\end{lemma}
\begin{proof}[Proof Sketch]
    We compare the algorithm's selected action \( a_t \) against the fixed benchmark \( \boldsymbol{u}^* \), which yields a per-round reward difference and a cost-weighted deviation. This motivates a regret decomposition that includes an arm selection gap and a dual-weighted cost gap. 
\begin{align}
&\text{UCB}_{r,t}^{\top} \boldsymbol{u}^* - \text{UCB}_{r,t}(a_t) + \frac{\hat{Q}_t}{B} \boldsymbol{\lambda}_t^{\top} \left( \frac{B}{\hat{Q}_t} \boldsymbol{s} - \boldsymbol{z}_t^{\text{conf}}(\boldsymbol{u}^*) \right) \nonumber
\\&\leq 
\frac{\hat{Q}_t}{B} \boldsymbol{\lambda}_t^{\top} \left( \frac{B}{\hat{Q}_t} \cdot \boldsymbol{s} - \boldsymbol{z}_t^{\text{conf}}(a_t) \right).\nonumber
\end{align}

Rewriting and summing over \( t = 1 \) to \( \tau - 1 \), and multiplying by \( q_t \). By defining a sequence of convex functions, i.e. $f_t(\boldsymbol{\lambda})$ encoding the dual-adjusted budget residual, we frame the second term as a regret bound in online convex optimization. Applying standard OGD guarantees over the simplex \( \mathcal{S} := \{ \boldsymbol{\lambda} \in \mathbb{R}^d : \| \boldsymbol{\lambda} \|_1 \leq T^{1/4},\ \boldsymbol{\lambda} \geq \textbf{0}_d \} \), we bound the cumulative difference between the algorithm's decisions and the best fixed dual variable in hindsight. This yields the desired upper bound.
\end{proof}
\begin{lemma}[Budget Residual Lower Bound]\label{lem:budget-residual}
At any stopping time \( \tau \leq T \), with probability at least \( 1 - 3KT\delta \), it holds that:
\begin{equation}
    \sum_{t=1}^{\tau - 1} R_t 
\geq 
\sum_{t=1}^{\tau - 1} q_t \cdot \frac{\widehat{Q}_t}{B} \left( \frac{B}{\widehat{Q}_t} - \frac{B}{Q} \right) \| \boldsymbol{\lambda}_t \|_1.
\end{equation}
\end{lemma}
\begin{proof}[Proof Sketch]
Applying Lemma~2 ensures that with high probability, the lower confidence bound of packing cost vector in $R_t$ is upper bounded by the true cost vector. We first decompose the residual \( R_t \) by separating \( \frac{B}{\widehat{Q}_t} \mathbf{1}_{d_1} - \left( \boldsymbol{\mathrm{LCB}}_{c,t}^{[1:d_1]} \right)^\top \boldsymbol{\mu}^*
 \) into \( \left(\frac{B}{\widehat{Q}_t} - \frac{B}{Q} \right) \mathbf{1}_{d_1} + \left(\frac{B}{Q} \mathbf{1}_d - \left( \boldsymbol{\mathrm{LCB}}_{c,t}^{[1:d_1]} \right)^\top \boldsymbol{\mu}^*
 \right) \), and bound the latter using the definition of \( \boldsymbol{u}^* \) and non-negativity of all terms. The covering constraint analysis follows similarly by symmetry. Then by dropping non-negative slack terms and summing across \( t = 1 \) to \( \tau - 1 \), we obtain the desired lower bound involving the norm \( \| \boldsymbol{\lambda}_t \|_1 \). 
\end{proof}

\begin{lemma}[Budget Violation Decomposition and Upper Bound]
Assume that at some stopping time \( \tau \leq T \), there exists a budget dimension \( j_0 \in \mathcal{I}_{\text{pack}} \) such that:
$\sum_{t=1}^{\tau - 1} q_t z_t^{(j_0)}(a_t) > B$.
Define the dual variable as \( \boldsymbol{\lambda} = \alpha \boldsymbol{e}_{j_0} \), with \( \alpha \in [0, T^{1/4}] \). Then, with probability at least \( 1 - 3KT\delta \), it holds that:
\begin{align*}
    &\sum_{t=1}^{\tau - 1}f_t(\boldsymbol{\lambda})
    = \sum_{t=1}^{\tau - 1}f_t(\alpha \boldsymbol{e}_{j_0}) \leq \alpha \bigg( Q_{\tau - 1} - Q + \frac{Q}{B} \bar{q} \\&+ \frac{Q}{B} \sum_{t=1}^{\tau - 1} q_t \left( z_t^{(j_0)}(a_t) - \left[ \boldsymbol{\mathrm{LCB}}_{c,t}(a_t) \right]^{(j_0)} \right) \bigg) \\
    &\quad + \sum_{t=1}^{\tau - 1} q_t \alpha \cdot \frac{\widehat{Q}_t}{B} \left( \frac{B}{\widehat{Q}_t} - \frac{B}{Q} \right)
    + O\left( \alpha \left( \frac{1}{Q} + \frac{1}{B} \right) \sum_{t=1}^{\tau - 1} q_t \varepsilon_t \right),
\end{align*}
\end{lemma}

\begin{proof}[Proof Sketch]
We consider a stopping time \( \tau \le T \) such that a budget constraint is violated in coordinate \( j_0 \in \mathcal{I}_{\text{pack}} \), i.e., $\sum_{t=1}^{\tau - 1} q_t z_t^{(j_0)}(a_t) > B.$
Letting \( \boldsymbol{\lambda} = \alpha \boldsymbol{e}_{j_0} \) for some \( \alpha \in [0, T^{1/4}] \), we analyze the cumulative dual-adjusted cost residual. We decompose the expression into:  
(i) the normalization error from \( \frac{B}{\widehat{Q}_t} - \frac{B}{Q} \),  
(ii) the empirical budget violation \( \sum_t q_t \left( z_t^{(j_0)}(a_t) - \left[ \boldsymbol{\mathrm{LCB}}_{c,t}(a_t) \right]^{(j_0)} \right) \),  
and (iii) prediction errors \( \varepsilon_t \).  Using standard concentration bounds and boundedness assumptions, we derive the stated upper bound.
\end{proof}

Then the following theorem provides a high-probability regret upper bound for Algorithm 1:
\begin{theorem}[High-Probability Regret Bound under Time-Varying Demand]\label{thm:combined-regret}
With probability at least $1 - 3KT\delta$, the algorithm satisfies the following regret bound:
\begin{align*}
    &\mathrm{OPT}_{\mathrm{LP}} - \sum_{t=1}^{\tau - 1} q_t R_t 
\le O \bigg( \log\left( \frac{1}{\delta} \right) 
\Big( \mathrm{OPT}_{\mathrm{LP}} \sqrt{ \frac{\bar{q} K}{B} } 
\\&\quad+ \sqrt{ \bar{q} K \cdot \mathrm{OPT}_{\mathrm{LP}} } \Big) 
+ \left( \frac{1}{Q} + \frac{1}{B} \right) \sum_{t=1}^{\tau - 1} q_t \epsilon_t + M T^{3/4} \bigg),
\end{align*}
where $Q = \sum_{t=1}^T q_t$, $\bar{q} = \max_t q_t$, and $\epsilon_t = |\hat{Q}_t - Q|$.
\end{theorem}

\begin{proof}[Proof Sketch]
We consider two cases depending on whether any budget constraint is violated before horizon $T$.

\textbf{Case 1: $\tau \le T$ (some budget is violated before the horizon ends).} In this case, 
\begin{equation}
\sum_{t=1}^{\tau - 1} q_t \cdot \boldsymbol{\mathrm{UCB}}_{r,t}^\top \boldsymbol{u}^* 
\geq \sum_{t=1}^{\tau - 1} q_t \boldsymbol{r}_t^\top \boldsymbol{u}^* 
= \mathrm{OPT}_{\mathrm{LP}} \cdot \frac{Q_{\tau - 1}}{Q},\nonumber
\end{equation}

We begin by decomposing the regret between the LP benchmark and the cumulative reward of the algorithm:
\begin{equation}
    \mathrm{OPT}_{\mathrm{LP}} \cdot \frac{Q_{\tau-1}}{Q} - \sum_{t=1}^{\tau-1} q_t \boldsymbol{\mathrm{UCB}}_{r,t}(a_t). \nonumber
\end{equation}
Choosing \( \boldsymbol{\lambda} = \mathrm{OPT}_{\mathrm{LP}}/Q \cdot \boldsymbol{1} \le 1 \), and combining Lemmas~3–5, and rearranging terms, we can derive the results of Theorem 1.

\textbf{Case 2: $\tau = T+1$ (no packing budget is violated and all rounds are used).}
With probability $\ge 1 - 3KT\delta$, we have 
\begin{equation}
    \sum_{t=1}^{\tau - 1} q_t \cdot \boldsymbol{\mathrm{UCB}}_{r,t}^\top \boldsymbol{u}^* 
\geq \sum_{t=1}^{\tau - 1} q_t \boldsymbol{r}_t^\top \boldsymbol{u}^* =  \mathrm{OPT}_{\mathrm{LP}} \nonumber
\end{equation}

By primal-dual optimality, we know that
\[
\sum_{t=1}^{\tau - 1} q_t \boldsymbol{r}_t^\top \boldsymbol{u}^* = \mathrm{OPT} \quad \text{and} \quad \sum_{t=1}^{\tau - 1} q_t \boldsymbol{z}_t^\top \boldsymbol{u}^* \le B.
\]
Using Lemma~3 to control the gap between $\boldsymbol{\mathrm{UCB}}_{r,t}(a_t)$ and $\boldsymbol{\mathrm{UCB}}_{r,t}^\top \boldsymbol{u}^*$, and Lemma~4 with $\boldsymbol{\mu} = \boldsymbol{0}$, we obtain:
\[
\mathrm{OPT} - \sum_{t=1}^{\tau - 1} q_t \boldsymbol{\mathrm{UCB}}_{r,t}(a_t)
\le O\left( \frac{1}{Q} \sum_{t=1}^{\tau - 1} q_t \varepsilon_t + M T^{3/4} \right).
\]
Further combining two cases, we arrive at the final high-probability regret bound.
\end{proof}

The following theorem provides a high-probability constraint violation upper bound for Algorithm~1:
\begin{theorem}
\label{thm-upperbound-1}
Consider the COPAC-UCB algorithm, that is provided with predictions that satisfy $|\widehat{Q}_t - Q| \le \varepsilon_t$ for all $t \in [\tau - 1]$. With probability $\ge 1 - 3KTd \delta$, for all $i \in \mathcal{I}_{\text{cover}}$,
\begin{align*}
    &\max\left( 0,\ B - \sum_{t=1}^{\tau -1} q_t z_t^{(i)}(a_t) \right)
    \\&\le O\left( T^{-1/4} Q + M T^{3/4}
    + \frac{1}{B} \sum_{t=1}^{\tau -1} q_t \varepsilon_t
    + \frac{Q}{B} \sqrt{T} \right).
\end{align*}
\end{theorem}

\begin{proof}[Proof Sketch]
Fix \( i \in \mathcal{I}_{\text{cover}} \) such that \( [\boldsymbol{\lambda}]_i = T^{1/4} \) and all other coordinates are zero. Lemma~3 can then be rewritten as:
\begin{align}
    & \sum_{t=1}^{\tau-1} \left[ \frac{q_t \widehat{Q}_t}{B} \boldsymbol{\mathrm{UCB}}_{c,t}^{(i)}(a_t) - \frac{q_t \widehat{Q}_t}{B} \cdot \frac{B}{\widehat{Q}_t} \right] + O(M T^{3/4}) \nonumber \\
    &\quad \ge \sum_{t=1}^{\tau-1} \bigg( q_t \langle \boldsymbol{\mu}^*, \boldsymbol{r}_t \rangle - q_t \cdot \boldsymbol{\mathrm{UCB}}_{r,t}(a_t)
    + \frac{q_t \widehat{Q}_t}{B} \cdot \lambda_t \nonumber\\&\qquad\qquad- \lambda_t q_t 
    + \frac{q_t \widehat{Q}_t}{B} \lambda_t \langle \boldsymbol{\mu}^*, \boldsymbol{z}_t^{[\mathcal{I}_{\text{cover}}]} \rangle \bigg).
    \label{lemma3_rewritten}
\end{align}

For the left-hand side of \eqref{lemma3_rewritten}, we can upper bound:
\begin{align}
    &\sum_{t=1}^{\tau -1} \frac{q_t \widehat{Q}_t}{B} \left[ \boldsymbol{\mathrm{UCB}}_{c,t}^{(i)}(a_t) - \frac{B}{\widehat{Q}_t} \right]\nonumber
    \\&\ge \sum_{t=1}^{\tau -1} \Big( q_t - \frac{q_t Q}{B} z_t^{(i)}(a_t)
    + \frac{Q}{B} O(\| \boldsymbol{\mathrm{LCB}}_{c,t}^{(i)} - z_t^{(i)} \|)\\
    &+ O\left( \frac{\bar{q} \varepsilon_t}{B} \right) \Big) \nonumber \\
    &= \sum_{t=1}^{\tau -1} \left( \frac{Q}{B} O(\| \boldsymbol{\mathrm{LCB}}_{c,t}^{(i)} - z_t^{(i)} \|)
    + O\left( \frac{\bar{q} \varepsilon_t}{B} \right) \right)\nonumber
    \\&+ \frac{Q}{B} \left[ B - \sum_{t=1}^{\tau -1} q_t z_t^{(i)}(a_t) \right].
    \label{bound_1}
\end{align}

For the right-hand side, we have:
\begin{align}
   & - T^{-1/4} \left( q_t \langle \boldsymbol{\mu}^*, \boldsymbol{r}_t \rangle - q_t \cdot \boldsymbol{\mathrm{UCB}}_{r,t}(a_t) \right)\nonumber
    \\&\ge - {T^{-1/4}} q_t \boldsymbol{\mathrm{UCB}}_{c,t}^{(i)}(a_t)
    \ge - T^{-1/4}{q_t}.
    \label{bound_2}
\end{align}

The term $\lambda_t q_t - \frac{q_t \widehat{Q}_t}{B} \lambda_t \langle \boldsymbol{\mu}^*, \boldsymbol{z}_t^{[\mathcal{I}_{\text{cover}}]} \rangle$ can be bounded as:
\begin{align}
    &\lambda_t q_t - \frac{q_t \widehat{Q}_t}{B} \lambda_t \langle \boldsymbol{\mu}^*, \boldsymbol{z}_t^{[\mathcal{I}_{\text{cover}}]} \rangle\nonumber
    \\&\ge \lambda_t q_t \cdot \frac{\varepsilon_t}{B} \langle \boldsymbol{\mu}^*, \boldsymbol{z}_t^{[\mathcal{I}_{\text{cover}}]} \rangle
    = O\left( \frac{1}{B} \sum_{t=1}^{\tau -1} q_t \varepsilon_t \right).
    \label{bound_3}
\end{align}

Combining Lemma 2, \eqref{bound_1}, \eqref{bound_2}, and \eqref{bound_3}, we derive the violation bound for all $i \in \mathcal{I}_{\text{cover}}$.
\end{proof}

\textbf{Remark:} Theorem~1 (regret bound) and Theorem~2 (constraint violation bound) both reveal the key factors contributing to learning performance. In particular, the estimation error $\epsilon_t$ from demand forecasting directly affects the regret and constraint violation through the predicted task volume $\hat{Q}_t$. The $\widetilde{\mathcal{O}}(M^{3/4})$ term arises from the regret of the OGD used to update the dual variables. The remaining terms are governed by the confidence bounds for $\boldsymbol{r}$ and $\boldsymbol{z}$. These components jointly reflect how uncertainty in system dynamics and predictive advice influences online decision-making.

\section{Experiment}

\subsection{Experimental Setup}

\textbf{Collaborative inference setting:} We evaluate the proposed online constrained offloading algorithm in a collaborative inference system with four candidate LLM families:
\texttt{Gemma2\_2b}, \texttt{Llama3.2\_1b}, \texttt{Qwen2.5\_0.5b}, and \texttt{Qwen2.5\_1.5b}.
At each decision round $t$, the system receives a batch of $q_t$ tasks and selects an inference option (i.e., one of the candidate LLMs, or its quantized variant) to serve the tasks.
Each option exhibits heterogeneous and uncertain trade-offs in (i) task accuracy, (ii) inference latency (in seconds), (iii) cost per 1k tokens (USD), and (iv) the average response length (in tokens), which together induce time-varying constrained decision making.

\textbf{Trace-driven model profiling from edge measurements:} Rather than relying on hand-crafted or synthetic model profiles, we build a trace-driven simulator grounded in real measurements reported in \cite{husom2025sustainable}. The measurement traces provide dataset-specific accuracy, end-to-end inference latency, and output-length statistics for a diverse set of quantized variants drawn from commonly deployed lightweight LLM families. We use these traces to instantiate the stochastic reward-and-cost model in our evaluation: for each inference option, task correctness is modeled according to the measured accuracy, the per-request delay is sampled from the corresponding empirical latency distribution, and the token-level output length follows the measured response-length statistics. The monetary cost of each request is then computed under a unified cost-per-1k-tokens accounting rule based on the sampled output length. This trace-driven construction preserves realistic accuracy--latency--cost trade-offs observed in edge deployments, enabling a faithful evaluation of online decisions under uncertainty.


\textbf{Constraints:} We consider a budget-constrained and delay-sensitive deployment scenario reflecting practical service-level objectives. The system operates under a packing constraint of a maximum total monetary budget of \$8,000, which limits the total cumulative cost of serving all tasks. Concurrently, we impose a covering constraint on inference latency: at least 80\% of all tasks must be served within 180s, enforcing that more than half of the responses meet low-latency requirements typical for user-facing applications.

\textbf{Dynamic Workloads:}
We construct a dynamic task workload based on the HumanEval benchmark, which contains 164 code-generation tasks and is used without subsampling. At each round $t$, a batch of $q_t$ tasks arrives and must be served by the selected inference option. To capture both stationary and temporally correlated demand variations, we evaluate our algorithm under two arrival patterns: (i) i.i.d. arrivals, where $\{q_t\}$ are independently generated across rounds, and (ii) autoregressive arrivals following an AR(1) process, where the current demand depends on the previous round to emulate bursty or persistent traffic. 

Across all settings, we adopt a unified demand estimation strategy to approximate the total demand $\hat{Q}_t$. The estimator combines historical observations $\{q_s\}_{s=1}^{t-1}$ with forecasts of future demands $\{\hat{q}_s\}$ produced by standard time-series prediction methods. To reduce computational overhead, we employ a “power-of-two” update rule: $\hat{Q}_t$ is updated only when $t$ is a power of two (i.e., $t=2^k$ for some $k\in\mathbb{N}^+$); otherwise, the previous estimate is reused. When an update is triggered, $\hat{Q}_t$ is computed by aggregating the realized past demand and the predicted future demand, providing an efficient yet stable approximation that supports online decision making under time-varying workloads.


\subsection{Benchmark Algorithms}
We compare our COPAC-UCB algorithm with three representative baselines, all extended to handle both packing and covering constraints but without demand prediction:
\begin{itemize}
\item \textbf{AD-UCB:} A UCB-based algorithm for BwCR that formulates the problem as an online convex program and uses optimistic estimates to balance exploration and exploitation under convex constraints~\cite{agrawal2014bandits}. In our setting, AD-UCB handles both packing and covering constraints as linear constraints but lacks the ability to predict future demand fluctuations.

\item \textbf{PD-BwK:} A primal-dual algorithm for BwK that maintains resource prices via multiplicative updates and selects arms greedily based on estimated reward-to-cost ratios~\cite{badanidiyuru2018bandits}. Although efficient and competitive in stationary settings, it is conservative and adapts poorly to non-stationary or covering-constrained environments.

\item \textbf{SW-UCB:} SW-UCB extends the classical UCB framework to Non-Stationary Bandits with Knapsacks by employing a sliding-window estimator that focuses only on recent observations~\cite{liu2022non}. This enables SW-UCB to adapt to abruptly changing environments, but its short memory can lead to overly cautious resource utilization and suboptimal SLA compliance in predictable, gradually changing scenarios like ours.
\end{itemize}
These baselines collectively illustrate the trade-offs between optimism, conservatism, and short-term adaptability in constrained bandit learning.


\begin{table}[]
\caption{Collaborative Inference Symstem Parameters}
\begin{tabular}{|c|c|c|c|c|}
\hline
LLM           & \begin{tabular}[c]{@{}c@{}}Accuracy\\ (\%)\end{tabular} & \begin{tabular}[c]{@{}c@{}}Delay\\ (s)\end{tabular} & \begin{tabular}[c]{@{}c@{}}Cost (\$ \\ 1k tokens)\end{tabular} & \begin{tabular}[c]{@{}c@{}}Average\\ response length\end{tabular} \\ \hline
Gemma2\_2b    & 0.77                                                    & 281.92                                              & 0.005                                                          & 168.29                                                            \\ \hline
Llama3.2\_1b  & 0.84                                                    & 84.43                                               & 0.015                                                          & 124.67                                                            \\ \hline
Qwen2.5\_0.5b & 0.54                                                    & 41.05                                               & 0.001                                                          & 209.98                                                            \\ \hline
Qwen2.5\_1.5b & 0.27                                                    & 111.53                                              & 0.003                                                          & 130.84                                                            \\ \hline
\end{tabular}
\label{table:Inference_Symstem}
\end{table}



\begin{figure*}[th]
	\centering
	\subfigure[Estimation Error]{
		\begin{minipage}[b]{0.33\textwidth}
			\includegraphics[width=0.95\textwidth]{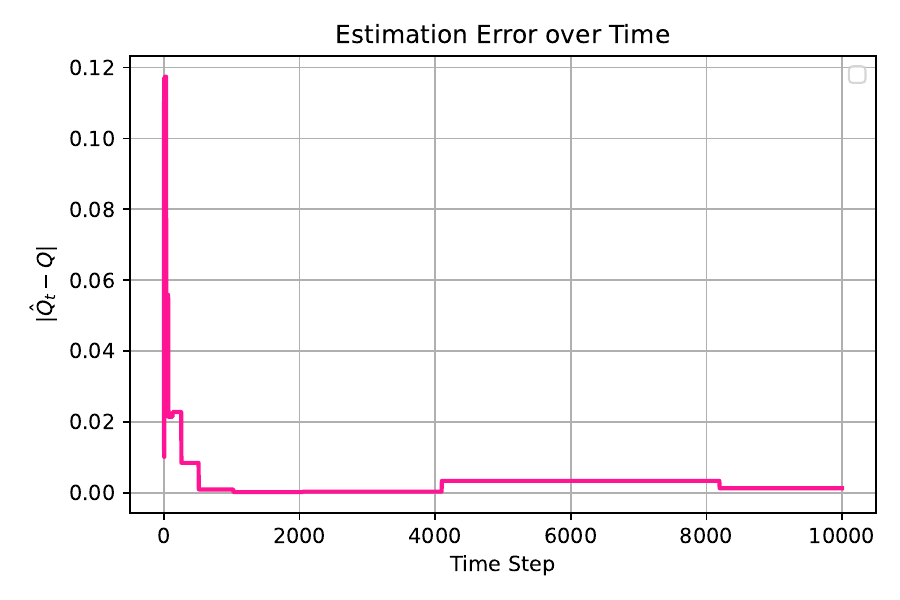}
		\end{minipage}
		\label{fig:time_predict_iid}
	}
	\hspace{-6mm}
    	\subfigure[Cumulative Regret]{
    		\begin{minipage}[b]{0.33\textwidth}
   		 	\includegraphics[width=0.95\textwidth]{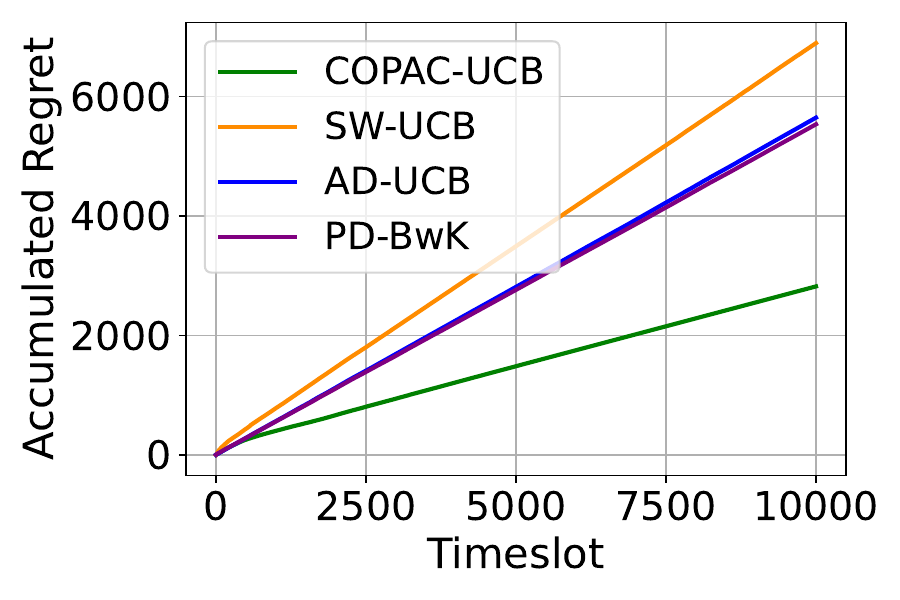}
    		\end{minipage}
		\label{fig:regret_iid}
    	}
    \hspace{-6mm}
    	\subfigure[Constraint Violation]{
		    \begin{minipage}[b]{0.33\textwidth}
   	 	    \includegraphics[width=0.95\textwidth]{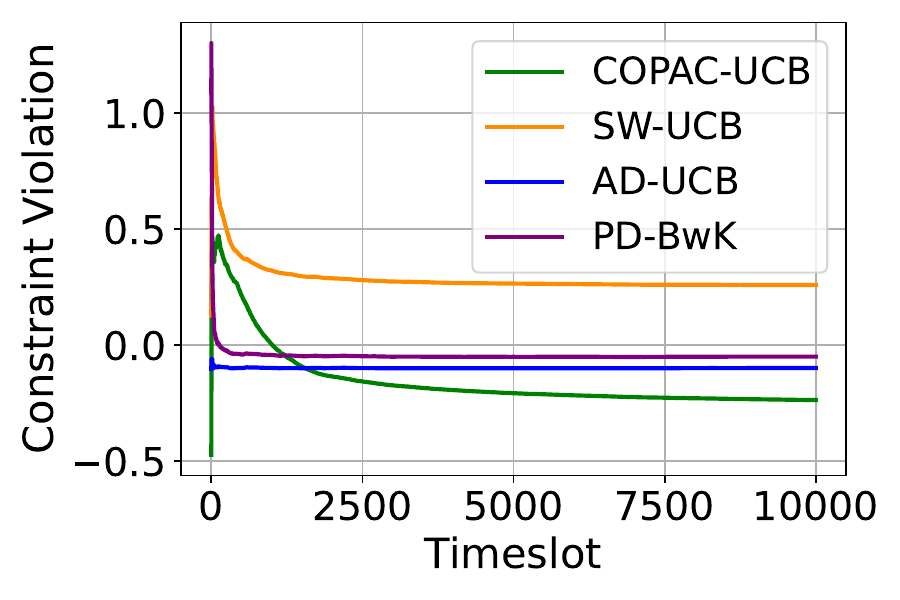}
		    \end{minipage}
	    \label{fig:reward_iid}
	    }
	\caption{Performance with i.i.d model.}
	\label{iid}
\end{figure*}

\begin{figure*}[th]
	\centering
	\subfigure[Estimation Error]{
		\begin{minipage}[b]{0.33\textwidth}
			\includegraphics[width=0.95\textwidth]{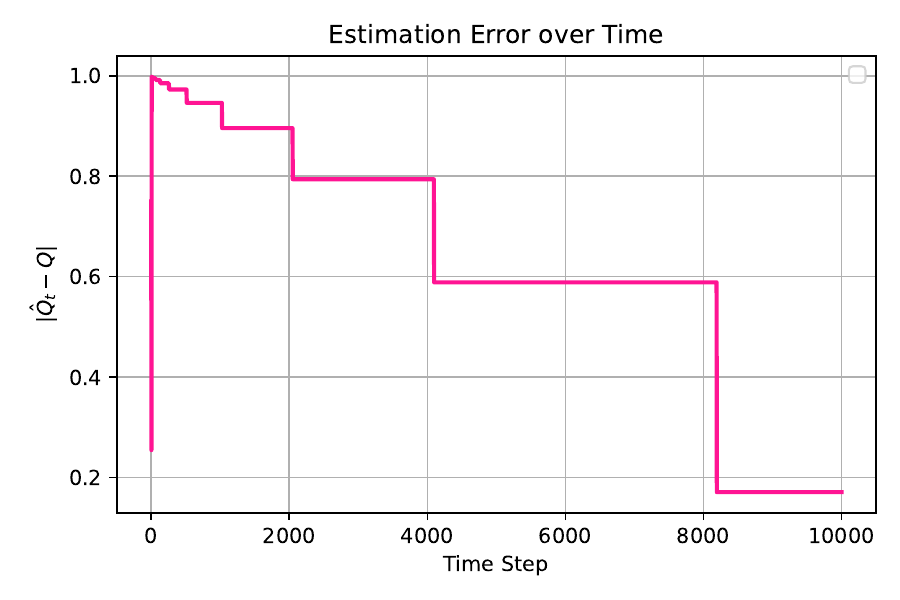}
		\end{minipage}
		\label{fig:time_predict}
	}
	\hspace{-6mm}
    	\subfigure[Cumulative Regret]{
    		\begin{minipage}[b]{0.33\textwidth}
   		 	\includegraphics[width=0.95\textwidth]{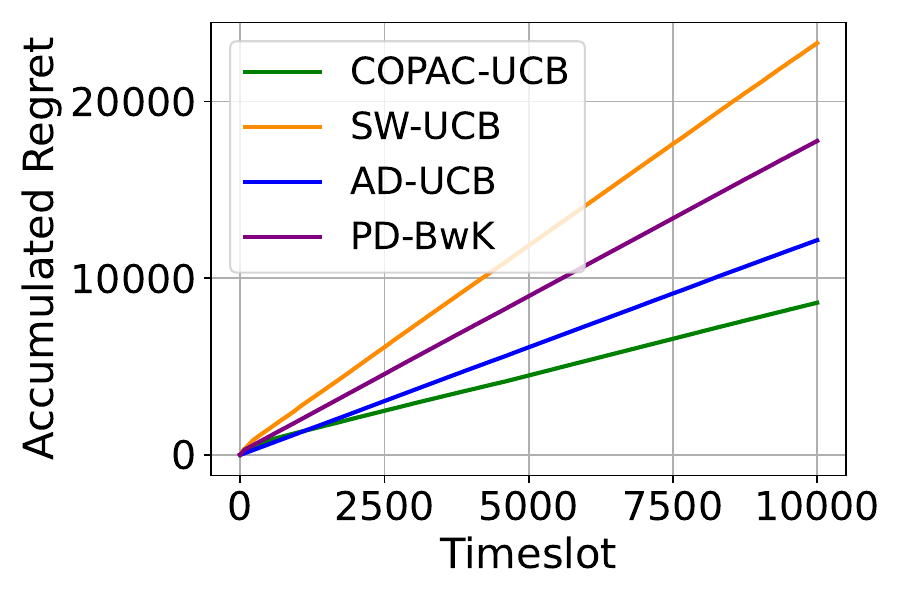}
    		\end{minipage}
		\label{fig:regret_linear}
    	}
    \hspace{-6mm}
    	\subfigure[Constraint Violation]{
		    \begin{minipage}[b]{0.33\textwidth}
   	 	    \includegraphics[width=0.95\textwidth]{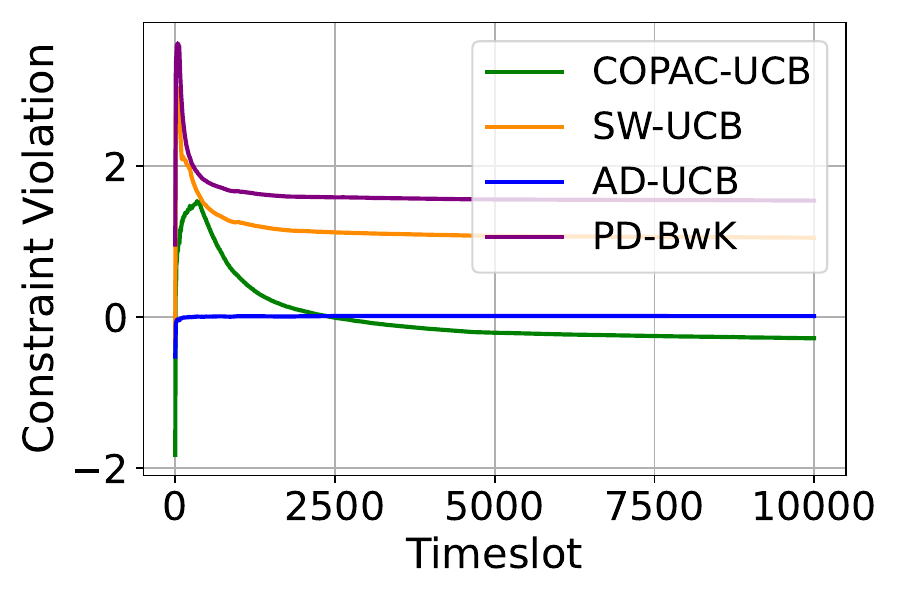}
		    \end{minipage}
	    \label{fig:reward_linear}
	    }
	\caption{Performance with AR (1) model.}
	\label{linear}
\end{figure*}

\subsection{Performance with i.i.d. Demand}

We first evaluate the algorithms under an i.i.d. demand setting, where each $q_t$ is independently drawn from a Gaussian distribution with mean 2 and variance 0.5. The experimental results are presented in Fig~2, which reports the estimation error of $Q$ over time (Fig.~2a), cumulative regret (Fig.~2b), and constraint violation (Fig.~2c).

In Fig.~2a, we observe that the estimation error of $Q$ decreases rapidly as time progresses, benefiting from the stability of the i.i.d. demand sequence. The power-of-two update policy allows the estimation to converge quickly and maintain low error with negligible computational overhead.

Fig.~2b compares the cumulative regret of the four algorithms. COPAC-UCB achieves the lowest regret throughout the horizon, highlighting the advantage of incorporating predictive estimation and explicitly optimizing for covering constraints. AD-UCB performs second best, leveraging its optimistic convex optimization design, though it lacks foresight and reacts more slowly to constraint trade-offs. PD-BwK exhibits higher regret due to its conservative, primal-dual updates, which underexploit available resources in a stationary setting. SW-UCB performs worst, accumulating the highest regret as its short-memory sliding-window mechanism discards useful long-term information and underutilizes resources unnecessarily.

Fig.~2c shows the average constraint violation over time. COPAC-UCB consistently achieves the smallest (and slightly negative) violation, indicating it not only satisfies the covering constraint but often slightly over-achieves it. AD-UCB maintains moderate violations, while PD-BwK remains more conservative, with higher violations than COPAC-UCB. SW-UCB again performs the worst, with persistently large violations due to its inability to exploit the stationary nature of the environment and its overly cautious allocation strategy.

Overall, these results confirm that in an i.i.d. setting, COPAC-UCB successfully balances exploration, exploitation, and constraint satisfaction through predictive demand estimation and constraint-aware decision-making, outperforming the baselines on both regret minimization and SLA compliance.

\subsubsection*{Performance with AR (1) Linear Demand}
We evaluate algorithm performance under a time-correlated AR(1) demand model, $q_t = \alpha + \beta q_{t-1} + \varepsilon_t$, with $\alpha = 2$, $\beta = 0.5$. Fig~3 summarizes the results: estimation error (Fig.~3a), cumulative regret (Fig.~3b), and constraint violation (Fig.~3c).

As shown in Fig.~3a, the estimation error of cumulative demand $Q$ remains higher and decays more slowly under AR(1) demand compared to the i.i.d. setting (Fig.~2a). This is because temporal correlation causes forecasting errors to propagate across rounds, one inaccurate prediction can distort subsequent estimates. As a result, all algorithms suffer from higher cumulative regret (Fig.~3b) and increased constraint violations (Fig.~3c). While COPAC-UCB continues to outperform the baselines in both cumulative reward and SLA satisfaction, its performance advantage narrows in the AR setting, especially during the early rounds. This is expected: time-correlated demand requires longer adaptation before accurate prediction stabilizes, and the dual update mechanism becomes more conservative. Interestingly, the performance gap between AD-UCB, PD-BwK, and SW-UCB also shrinks, suggesting that none of these baselines effectively exploit temporal structure. Their respective update rules, whether multiplicative (PD-BwK) or sliding-window-based (SW-UCB), lack mechanisms to anticipate future demand shifts, often leading to mismatches between allocation and actual task volume.

In summary, the i.i.d. scenario allows all algorithms (except SW-UCB) to achieve lower regret, likely with sublinear growth due to stationarity, and smaller constraint violations. In contrast, the AR scenario highlights the importance of prediction: COPAC-UCB retains an advantage by leveraging estimation, while the baselines degrade more uniformly.

\section{Conclusion}
We studied the problem of online LLM selection in an offloading system with time-varying task demand. Each LLM has unknown, stochastic accuracy, latency, and cost, and decisions must satisfy hard budget and soft service-level constraints. We modeled this as a stochastic bandit problem with mixed packing and covering constraints under non-stationary demand. Based on this, we proposed a learning-augmented online algorithm that combines confidence-based exploration with sequential prediction to balance reward and constraint satisfaction. We provided theoretical guarantees on regret and constraint violations, and validated our approach through experiments on both synthetic and real-world workloads. Future directions include handling richer prediction feedback, more complex constraints, and adversarial demand.


\bibliographystyle{IEEEtran}
\bibliography{references}

\end{document}